%% file: main.tex
\documentclass[10pt,conference]{IEEEtran}

\usepackage{amsmath, amsmath, amsthm, amssymb, braket, mathtools}

\usepackage[colorlinks = true,
            linkcolor = blue,
            urlcolor  = blue,
            citecolor = blue,
            anchorcolor = blue]{hyperref}
            
\usepackage{enumitem}
\usepackage{rotating}
\usepackage{tabulary}
\usepackage{multirow}
\usepackage{algorithm, algorithmic}

\usepackage{graphicx, epstopdf}
\usepackage{wrapfig}
\usepackage{caption}
\usepackage[labelformat=simple]{subcaption}

\input{macros}

\theoremstyle{remark}
\newtheorem*{remark}{Remark}

\begin{document}

\title{Active Learning for Community Detection in Stochastic Block Models}

\author{
\IEEEauthorblockN{
Akshay Gadde, 
Eyal En Gad, 
Salman Avestimehr and 
Antonio Ortega
}
\IEEEauthorblockA{
%Department of Electrical Engineering\\
University of Southern California, Los Angeles\\
Email: \{agadde,engad\}@usc.edu, avestimehr@ee.usc.edu, antonio.ortega@sipi.usc.edu 
}
}

\maketitle

\begin{abstract}
The stochastic block model~(SBM) is an important generative model for random graphs in network science and machine learning, useful for benchmarking community detection (or clustering) algorithms. The symmetric SBM generates a graph with $2n$ nodes which cluster into two equally sized communities. Nodes connect with probability $p$ within a community and $q$ across different communities. We consider the case of $p=a\ln (n)/n$ and $q=b\ln (n)/n$. In this case, it was recently shown that recovering the community membership (or label) of every node with high probability (w.h.p.) using only the graph is possible if and only if the Chernoff-Hellinger (CH) divergence $D(a,b)=(\sqrt{a}-\sqrt{b})^2 \geq 1$. In this work, we study if, and by how much, community detection below the clustering threshold (i.e. $D(a,b)<1$) is possible by querying the labels of a limited number of chosen nodes (i.e., active learning). Our main result is to show that, under certain conditions, sampling the labels of a vanishingly small fraction of nodes (a number sub-linear in $n$) is sufficient for exact community detection even when $D(a,b)<1$. Furthermore, we provide an efficient learning algorithm which recovers the community memberships of all nodes w.h.p. as long as the number of sampled points meets the sufficient condition. We also show that recovery is not possible if the number of observed labels is less than $n^{1-D(a,b)}$. The validity of our results is demonstrated through numerical experiments.

\end{abstract}

\section{Introduction}

Community detection (or clustering) is a fundamental problem in the study of networks~\cite{Fortunato-PR-10}. 
%Community detection (or clustering) is a fundamental problem in computer science and machine learning. 
In this problem, it is assumed that each node (or vertex) of a network belongs to one out of a few latent groups or communities, and that the topology of the network depends on these \emph{a priori} unknown group memberships (or labels).
The goal is to recover the communities by partitioning the nodes into different classes, with denser connections within classes and sparser connections between classes.
This problem arises in many areas,  such as  detecting protein  complexes in protein interaction networks~\cite{Chen-BIO-06} and image classification and segmentation~\cite{Shi-PAMI-00}.
%This problem arises in many areas. For example, in co-authorship networks, the communities may correspond to research topics; in protein interaction networks, the communities are functional modules of proteins etc.~\cite{Fortunato-PR-10}. 

In community detection problems, stochastic block models (SBM)~\cite{Holland-SN-83} are random graph models that are commonly used for benchmarking the performance of different clustering algorithms and deriving fundamental limitations of graph clustering. In this model, an edge exists between a pair of nodes independently with a probability that depends on the community membership of the two nodes. The simplest example of an SBM, which is the one considered in this paper, contains two communities of equal sizes, such that a pair of nodes from the same community are connected with probability $p$, and nodes of different communities are connected with probability $q$. The limits on the performance of different algorithms in terms of the relative difference between $p$ and $q$ have been established (see~\cite{Abbe-15} and references therein).

%Clustering algorithms perform better if the graph is more distinctly clustered, i.e., there are significantly more edges within communities compared to the number of edges between communities. Stochastic block model (SBM) is a generative model for random graphs with community structure~\cite{Holland-SN-83}, useful for benchmarking the performance of different clustering algorithms as a function graph's clusterdness. 
%
%While clustering algorithms have been studied extensively for a few decades, the fundamental limits of such algorithms, as a function of the graph's clusteredness, are only beginning to be established in the last few years~\cite{Abbe-15}. The recent progress is centered on the so-called stochastic block model (SBM).

In recent years, the fundamental limits of community detection in SBM (using any method) have been studied~\cite{Decelle-PRE-11}. In particular, an elegant threshold was discovered for exact recovery of labels (also known as strong consistency), which is defined as finding the correct community memberships of all nodes asymptotically almost surely (a.a.s.). In the case of an SBM with $2n$ nodes and two communities of equal sizes, having within-community edge probability $p=a\ln n/n$ and between-communities edge probability $q=b\ln n/n$ (where $a$ and $b$ are constants greater than $1$), exact recovery was shown to be possible if and only if $D(a,b)=(\sqrt{a}-\sqrt{b})^2\ge 1$~\cite{Abbe-IT-16}. We call $D(a,b)$ the Chernoff-Hellinger (CH) divergence, in accordance with~\cite{Abbe-15}. The condition $D(a,b)\ge 1$ can be written equivalently as $\Delta \geq \sqrt{M - 1/4}$ by defining 
\begin{equation}
M \coloneqq (a+b)/2 \text{ and } \Delta \coloneqq (a-b)/2
\end{equation}
so that $2M\ln{n}$ is the average degree of a node and $2\Delta \ln{n}$ is the average difference between the number of neighbors of a node in the same community and the number of neighbors in the opposite community.  
%differential degree of a node.
In addition, efficient algorithms were proposed that perform exact recovery wherever the above threshold condition holds~\cite{Abbe-15,Mossel-14}.

This paper is motivated by the following fundamental, and practically important question, to determine \emph{if, and by how much, community detection below the clustering threshold (i.e., $D(a,b)<1$) is possible by querying the labels of a 
vanishingly 
small fraction of carefully selected nodes (i.e., active learning)}. 
%
%AG: PRACTICAL MOTIVATION
The above question is essential in the context of recent trends of involving human workforce in data analytic problems (e.g., Amazon Mechanical Turk~\cite{Yan-ICML-11}). % for obtaining labeled data. 
In such setting, selecting the smallest number of most useful nodes for labeling is important for minimizing human effort.
%
%We answer the question in the affirmative 
We show that community detection below the threshold is possible by querying the labels of a vanishingly small fraction of nodes (i.e., $o(n)$ number of nodes).
We provide a bound on the number of observations (or samples) required as well as an algorithm which specifies which nodes to sample.

% An important question in this context is whether observing labels of a limited number of nodes can help overcome the detectability thresholds in SBM. 
% The problem has been studied in the case of sparse SBM, with $p = a/n$ and $q = b/n$. In sparse SBM, only partial recovery is expected, where the goal is to predict labels which have non-trivial correlation with true labels. It was shown that observing labels of a vanishingly small fraction of nodes \emph{randomly} does not affect the weak recovery but can lead to efficient local algorithms (in which the label prediction at each node depends only on its local neighborhood), whenever recovery is possible~\cite{Kanade-14,Mossel-TOC-16}

We consider an active learning framework in which an algorithm can choose which nodes to sample, up to a limited budget of $B(n)$ nodes. For simplicity, we restrict the budget to take the form $B(n)=n^s$, for a constant $s$. The main result of this paper is that exact recovery of the community membership is possible if 
%\textcolor{red}
{$4\Delta>1/3+\sqrt{4M+1/9}$ and $s\ge {1-\delta(a,b)}$, where $\delta(a,b)$ is as specified in Theorem~\ref{th:main_suf}. 
%= \frac{\left(2(a-b)-1/3-\sqrt{2(a+b)+1/9}\right)^2}{2\left((a+b)+2(a-b)-1/3-\sqrt{2(a+b)+1/9}\right)}$. 
Furthermore, we provide an efficient sampling algorithm which recovers the community memberships of all nodes a.a.s. wherever $s\ge 1-\delta(a,b)$. Additionally, we also show that recovery is not possible if the number of observed labels is less than $n^{1-D(a,b)}$.}

The proposed algorithm begins by obtaining an initial clustering using an approach presented in~\cite{Mossel-14}, which consists of spectral clustering followed by a replica method. We can characterize the nodes where this initial clustering is more likely to make mistakes.
%with high probability.
%
We propose to observe the labels of $B(n)$ nodes selected using this characterization, which allows us to correct the mistakes in the initial clustering. 
We show that observing {$\Omega(n^{1-\delta})$} labels in the proposed way and predicting the labels of the unsampled nodes according to the initial clustering guarantees that we a.a.s. recover all the labels. While most of the techniques used in the proof of our main result are based on previous works~\cite{Mossel-14,Abbe-15}, to the best of our knowledge, the proposed sampling algorithm is novel.

The problem of selecting a subset of nodes for labeling, which are most helpful in predicting the labels of the rest of the nodes is known as active learning on graphs~\cite{Settles-10,Gadde-KDD-14} in the machine learning literature. In a standard active learning setting, the nodes represent data points in a Euclidean space. A $k$-nearest neighbor graph is used to represent the similarities between them. 
Since the number of samples required (also called sampling complexity) by the proposed algorithm for exact recovery is
sub-linear (and thus, vanishingly small) in the number of nodes in the SBM, we suspect a similar approach to be useful in the above setting.
%
% We support this conjecture through experiments by applying the proposed approach for active learning in order to classify points drawn at random from two normally distributed clusters using the graph-based approach. The novel sampling algorithm outperforms state-of-the-art active learning algorithms in the considered model.

\noindent \textbf{Prior work.} The problem of community detection with
random sampling (as opposed to active learning)
in stochastic block models has been previously studied in the context of sparse SBMs, with $p = a/n$ and $q = b/n$~\cite{Zhang-PRE-14}. 
In sparse SBM, only weak recovery is expected, where the goal is to predict labels which have non-trivial correlation with true labels. It was shown in~\cite{Kanade-14,Mossel-TOC-16} that observing labels of a vanishingly small fraction of nodes \emph{randomly} does not affect the weak recovery asymptotically but can lead to efficient local algorithms (in which the label prediction at each node depends only on its local neighborhood), whenever recovery is possible.
The effect of sampling on exact recovery in SBM with log-degree, i.e., with $p = a \ln{n}/n$ and $q = b \ln{n}/n$ has not been studied. 
%In contrast, 
Our result in this paper demonstrates that active learning can largely impact community detection threshold in the regime of $p=a\ln n/n$ and $q=b\ln n/n$.

\section{Background and Main Result}
% AG: WE BEGIN WITH SOME BASIC DEFINITION
We begin with some basic definitions and notation.
\begin{definition}[Stochastic Block Model~\cite{Holland-SN-83}]
A stochastic block model is a random graph $G = (V,E) \sim \Gc(n,p,q)$ with $2n$ nodes constructed as follows: Select a labeling $\sigma\colon V \to \{1,-1\}$ for each node $v\in V$ independently and uniformly at random (i.e., $\Pr(\sigma_v = i) = 1/2, \; i \in \{1,-1\}$). Each pair of nodes $(v,u)$ is connected by an edge with probability $p$ if $\sigma_u = \sigma_v$ and with probability $q$ if $\sigma_u \neq \sigma_v$. 
\end{definition}
Community detection in SBM is equivalent to reconstruction of labels $\sigma$ (up to its sign) using $G$. The correctness of predicted labels is considered in an asymptotic sense, where $p$ and $q$ are allowed to be functions of $n$ and $n \to \infty$. We focus on the logarithmic degree regime, i.e., the case when $p_n = a \frac{\ln{n}}{n}$ and $q_n = b \frac{\ln{n}}{n}$, where $a$ and $b$ are constants. The notion of correctness, which is of most interest in this regime is called strong consistency or exact recovery.
%
% \begin{question}
% When can we recover $\sigma$ given $G$?
% \end{question}
%
%We consider the task of predicting the labeling $\sigma$. Let $\tau = \tau(G)$ be the predicted labeling.
%
\begin{definition}[Strong Consistency~\cite{Mossel-14}]
Given sequences $p_n$ and $q_n$,  the predicted labeling $\tau = \tau(G)$ is said to be strongly consistent if $\Pr(\tau = \sigma \text{ or } \tau = -\sigma) \to 1$ as $n \to \infty$.
\end{definition}

% \begin{definition}[Weak Consistency]
% The fraction of errors in $\tau$ is defined as 
% \begin{equation}
% \Delta(\sigma,\tau) = 1- \frac{1}{2n} \left| \sum_{i=1}^{2n} \sigma_i \tau_i \right|.
% \end{equation}
% Given sequences $p_n$ and $q_n$,  $\tau(G)$ is said to be weakly consistent if $\Delta(\sigma,\tau) \overset{P} \to 0$ as $n \to \infty$.
% \end{definition}
%
Note that strong consistency requires the number of errors to go to $0$. A weaker notion of correctness is also considered in the literature, called weak consistency (or partial recovery). This notion only requires that the fraction of errors goes to $0$, i.e., the number of errors is $o(n)$.

%\subsection{Clustering Threshold}
%Clustering threshold specifies the minimum difference between $a$ and $b$ required for exact recovery to be possible using any method. 
%
We next define differential degree, which will be useful for characterizing the fundamental limits on clustering and describing the sampling algorithm proposed in the next section. 
\begin{definition}[Differential Degree]
Differential degree of a node $v$ w.r.t. a labeling $\tau$ is defined as 
\begin{equation*}
d^*_{\tau}(v) = |\{u \in V| u \sim v, \tau_u = \tau_v\}| - |\{u \in V| u \sim v, \tau_u \neq \tau_v\}|.
\end{equation*}
$v$ is said to have a majority if $d^*_{\tau}(v) > 0$. Otherwise, $v$ has a minority.
\end{definition}
%
% \begin{definition}
% Let $X \sim \mathrm{Bin} (\frac{n}{2}, p)$ and $Y \sim \mathrm{Bin} (\frac{n}{2}, q)$. Then
% \begin{equation}
% P(n,p,q) = \Pr(Y \geq X).
% \end{equation}
% \end{definition}
%
Intuitively, any clustering algorithm, which groups the nodes such that there are more edges within communities and fewer edges between communities, can succeed only if each node belongs to the same community as the majority of its neighbors.
%, i.e., $d^*_\sigma(v) > 0$ for all $v \in V$. 
This necessary local condition also turns out to be sufficient for recovery.
%as stated in the following theorem. 
Specifically, it has been shown~\cite{Mossel-14} that a strongly consistent estimator for $\Gc(n, p_n, q_n)$ exists if and only if a.a.s. $d^*_{\sigma}(v) > 0 \;\; \forall \; v \in V$. This, in turn, holds if and only if $D(a,b)=(\sqrt{a}-\sqrt{b})^2 \ge 1$ or equivalently, $\Delta \geq \sqrt{M - 1/4}$ (hereafter we use $D$ as a shorthand for $D(a,b)$). 
The necessary part of the condition 
can be proven by analyzing the probability that a node $v$ has a minority, which can be computed using the fact that the number of neighbors of each node in the same (or different) community follows a binomial distribution with parameter $p = a \ln{n}/{n}$ (or $q = b \ln{n}/{n}$).

An algorithm to obtain a strongly consistent labeling under the assumption $D\ge 1$ is given in~\cite{Mossel-14}. This algorithm consists of two steps. The first step (spectral clustering followed by replica method) gives a weakly consistent labeling and the final step flips the labels of the minority nodes (with respect to the prediction in the previous step) to ensure that every node is a majority.
The first step of the algorithm is \emph{stable} in the sense that it works even when $D \leq 1$ and only requires $a > b$ to guarantee weak consistency. 
%
% \begin{theorem}[Necessary and sufficient condition for weak consistency]
% \label{th:weak}
% There exists a weakly consistent estimator for $\Gc(2n, p_n, q_n)$ if and only if 
% \begin{align}
% P(n,p_n,q_n) \to 0 \; &\Leftrightarrow \; \text{a.a.s. } d^*_{\sigma}(v) \leq 0 \; \text{for at most} \; o(n) \; \text{nodes} \nonumber \\ 
% & \Leftrightarrow \; a > b.
% \end{align}
% \end{theorem}

The fundamental question considered in this paper is whether one can get a strongly consistent estimator for the labels even when $D \leq 1$, by observing labels of a vanishingly small fraction of carefully selected nodes (in addition to $G$). In order to answer this question, we need to specify the number of samples needed as well as the locations of those samples.
To do this, we first define critical differential degree as follows:
\begin{definition}[Critical Differential Degree] Critical differential degree $\ell_\text{critical}$ is defined as the smallest number such that a.a.s. no node $v$ has $d^*_\sigma(v) \leq -\ell_\text{critical}$, i.e., 
\[\ell_\text{critical} = \inf\{\ell \mid \Pr(d^*_\sigma(v) \leq -\ell) = o(n^{-1})\}.\]
\end{definition}
\noindent Based on this definition, following theorems give the sufficient and necessary conditions on the number of samples needed. 
\begin{theorem}[Sufficient Condition]
\label{th:main_suf}
Sampling the labels of a subset of nodes $S \coloneqq \{v \mid d^*_\sigma(v) \leq \ell_\text{critical}\}$ is sufficient for exact recovery in $\Gc(n, a\frac{\ln{n}}{n}, b\frac{\ln{n}}{n})$. Moreover, 
assuming \hbox{$4\Delta> 1/3 \!+\!\sqrt{4M\!+\!1/9}$}, $|S|\leq  n^s$
%Let $B(n) = n^s$ be the number of observed labels.
%such that $s=\lim_{n\to\infty}\frac{\ln (B_n)}{\ln (n)}$.
%Then, exact recovery is possible for $\Gc(n, a\frac{\ln{n}}{n}, b\frac{\ln{n}}{n})$ if $2(a-b)>1/3+\sqrt{2(a+b)+1/9}$ and
if $s \ge 1-\delta(M,\Delta)$, where
\begin{equation}
\delta(M,\Delta) = \frac{\left(4\Delta-1/3-\sqrt{4M+1/9}\right)^2}{2\left(2M+4\Delta-1/3-\sqrt{4M+1/9}\right)}.
\end{equation}
% \[\delta(a,b) = \frac{\left(2(a-b)-1/3-\sqrt{2(a+b)+1/9}\right)^2}{2\left((a+b)+2(a-b)-1/3-\sqrt{2(a+b)+1/9}\right)}.\]
\end{theorem}
%
% In order to simplify the above expression, we define $M \coloneqq (a+b)/2$ and $\Delta \coloneqq (a-b)/2$ so that $2M\ln{n}$ is the average degree and $2\Delta \ln{n}$ is the average differential degree of a node. With this notation, $\delta(a,b)$ can be rewritten as:
% \begin{equation}
% \delta(M,\Delta) = \frac{\left(4\Delta-1/3-\sqrt{4M+1/9}\right)^2}{2\left(2M+4\Delta-1/3-\sqrt{4M+1/9}\right)}.
% \end{equation}
The above expression shows that the number of samples needed is smaller as $\Delta$ increases relative to $M$.
\begin{theorem}[Necessary Condition]
\label{th:main_nec}
Let $B(n) = n^s$ be the number of observed labels.
%such that $s=\lim_{n\to\infty}\frac{\ln (B_n)}{\ln (n)}$.
Then, exact recovery is not possible for $\Gc(n, a\frac{\ln{n}}{n}, b\frac{\ln{n}}{n})$ if  
$s\leq 1-D(a,b)$.
\end{theorem}
%

%\textcolor{red}{
\begin{remark}
Note that Theorem~\ref{th:main_nec} is congruous with previous results~\cite{Abbe-15,Mossel-14}. If $D \leq 0$ then it is not possible to find even a weakly consistent estimator and the number of samples required for exact recovery is linear in $n$. On the other hand, if $D \geq 1$, then exact recovery is possible without sampling. 
%If $0<D<1$, then the number of samples needed for exact recovery is sub-linear in $n$.
\end{remark}
While exact recovery is not possible without sampling when $D < 1$ (or equivalently $\Delta < \sqrt{M-1/4}$), Figure~\ref{fig:delta_vs_D} shows that the number of samples needed (as given by Theorem~\ref{th:main_suf}) is sub-linear in $n$ provided $\Delta \gtrsim 0.55\sqrt{M-1/4}$ (or $D \gtrsim 0.3$). 
Thus, exact recovery is possible in this regime by observing the labels of a vanishingly small fraction of the nodes. The figure also shows that, for a fixed value of $\Delta / \sqrt{M-1/4}$, the number of samples needed decreases as $M$, which specifies the total number of edges, increases. 
%
%As $D$ gets closer to $0$, clustering becomes more difficult. Therefore, we need to sample more nodes. 
%}
\begin{figure}
\centering
\includegraphics[width=0.35\textwidth]{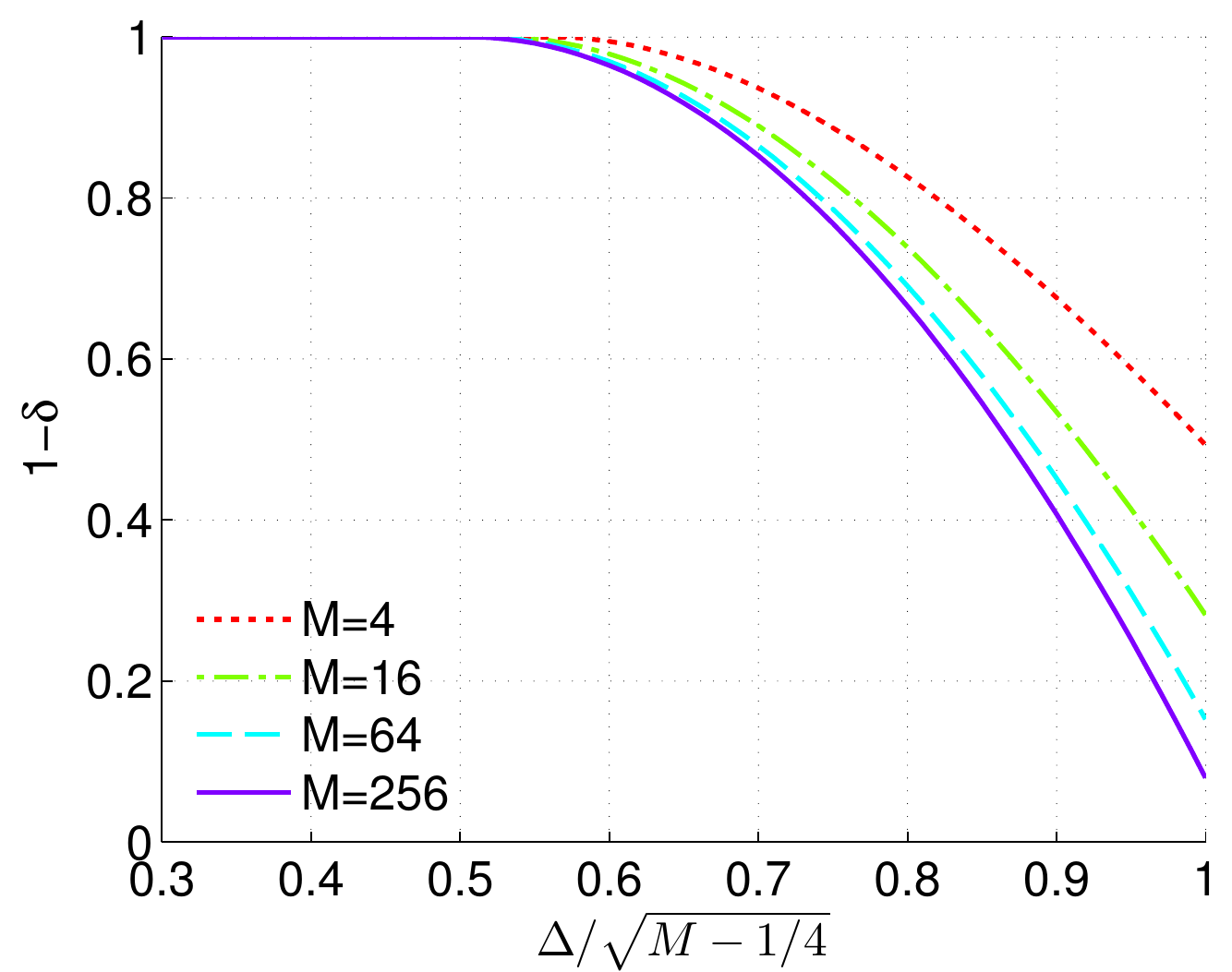}
\caption{$1-\delta$ vs. $\Delta/\sqrt{M - 1/4}$ for different values of $M$.}
%where $n^{1-\delta}$ is the number of samples needed according to Theorem~\ref{th:main_suf}. }
\vspace{-1\baselineskip}
\label{fig:delta_vs_D}
\end{figure}

The next section proves the sufficient condition in Theorem~\ref{th:main_suf} by providing an algorithm for selecting the subset of nodes to be sampled, and shows that sampling $n^{1-\delta}$ nodes using the proposed algorithm is enough for exact recovery. The proof of the necessary condition is provided in Section~\ref{sec:nec_cond}.
%The next two sections are dedicated to the proofs of sufficient and necessary condition in Theorem~\ref{th:main} respectively.
% \paragraph*{Idea}
% Sample the nodes that are wrongly labeled by the weakly consistent estimator. 
% \begin{enumerate}
% \item Characterization of the wrongly labeled nodes: Only nodes with $d^*_\sigma(v) \leq \kappa$ are wrongly labeled.

% \item Sample the nodes with $d^*_\tau(v) \leq \kappa + \text{const}$.
% \end{enumerate}

%%%%%%%%%%%%%%%%%%%%%%%%%%%%%%%%%%%%%%%

% \section{ProofMain Result}
% \label{sec:main_result}

\section{Sufficient Condition }
\label{sec:suf_cond}
% TO DO:\\
% 1. changes for iid model in Prop. 2 (Bernstein for community size)\\
%
In this section, we give a constructive proof for the sufficient condition in Theorem~\ref{th:main_suf} by providing an algorithm for sampling. 
We show that the sampling budget given by Theorem~\ref{th:main_suf} is sufficient for exact recovery using the proposed algorithm. 
% the number of samples needed by the proposed algorithm for exact recovery equals the smallest required sampling budget as given by Theorem~\ref{th:main_suf}.
%
The main idea behind the proposed algorithm is as follows: even when $D<1$, it is possible to get a weakly consistent estimator $\tau'$ for the labels (using Algorithm~$1$ from \cite{Mossel-14}), which makes only $o(n)$ mistakes as long as $D>0$. If we can characterize the nodes which are wrongly labeled by the weakly consistent estimator and then correct the mistakes by sampling their labels, we will have exact recovery. It can be shown that w.h.p. the mistakes in $\tau'$ are restricted to nodes which have small differential degrees. Thus, the proposed algorithm samples nodes with smallest differential degrees $d^*_{\tau'}(v)$ (see Algorithm~\ref{alg:samp_algo}).      
\begin{algorithm}
\caption{Sampling for exact recovery}
\label{alg:samp_algo}
\begin{algorithmic}[1] 
\REQUIRE Graph $G = (V,E)$, Sampling budget $B(n)$
\STATE Find $\tau'$ (an initial guess for the labels) by clustering the nodes using Algorithm~$1$ from~\cite{Mossel-14}.
\STATE Compute differential degrees $d^*_{\tau'}(v)$ for all $v \in V$.
%\STATE Sort the nodes in ascending order of $d^*_\tau(v)$
\STATE Sample $B(n)$ nodes with the smallest $d^*_{\tau'}(v)$.
\STATE Update $\tau'$ using the sampled labels to get $\tau$.
%\STATE For $i\in\{-1,1\}$, let $U_i$ be the set of nodes $v\in V$ for which $\tau''_v=i$.
%\STATE For each node $v\in V$, assign $\tau_v$ such that $v$ has at least as many neighbors in $U_{\tau''_v}$ as it has in $V\setminus U_{\tau''_v}$ (majority vote).
\end{algorithmic}
\end{algorithm}
\begin{theorem}
\label{th:suf}
Let $\tau$ be the labeling given by Algorithm~\ref{alg:samp_algo} after sampling $B(n) = n^s$ nodes.
%Let $s=\lim_{n\to\infty}\frac{\ln (B_n)}{\ln (n)}$.
Then $\tau$ is strongly consistent if $4\Delta>1/3+\sqrt{4M+1/9}$ and \hbox{$s\geq1-\delta(M,\Delta)$}.
\end{theorem}
The proof of Theorem~\ref{th:suf} follows by extending and combining several results and techniques from~\cite{Mossel-14}. We begin by characterizing the set of nodes, where the weakly consistent estimator $\tau$ makes mistakes.
\begin{definition}
For a positive real number $\epsilon$, let $V_{\epsilon}$ be the elements of $V$ that either have a differential degree $d^*_\sigma(v)$ less than $\epsilon\sqrt{a} \ln n$ or have more than $100 np$ neighbors.
\end{definition}

\begin{proposition}[Proposition 4.3 in~\cite{Mossel-14}]
\label{pro:correct}
For any $\epsilon>0$, the initial labeling $\tau'$ in Algorithm~\ref{alg:samp_algo} a.a.s. correctly labels every node in $V\setminus V_{\epsilon}$.
\end{proposition}
\begin{proposition}[Proposition~4.7 in \cite{Mossel-14}]
For sufficiently small $\epsilon$, a.a.s. no two nodes in
$V_\epsilon$ are adjacent.
\end{proposition}
The above proposition ensures that a.a.s. $|d^*_{\tau'}(v)| = |d^*_{\sigma}(v)|$ for $v \in V_\epsilon$. 
Our idea is to find $\ell_\text{critical}$ such that %$\Pr\left(d^*_\sigma(v) \leq -\ell_\text{critical}\right) = o(n^{-1})$. This would imply that 
a.a.s. no node has differential degree (w.r.t. $\sigma$) less than $-\ell_\text{critical}$. 
Since $|d^*_{\tau'}(v)| = |d^*_{\sigma}(v)|$ for all $v\in V_\epsilon$, sampling nodes with $d^*_{\tau'}(v) \leq \ell_\text{critical}$ will capture all the nodes in $V_\epsilon$ with $d^*_{\sigma}(v) \leq \ell_\text{critical}$. 

\begin{proposition}
%{\small
$\ell_\text{critical} = \left(\sqrt{4M+1/9} + 1/3 - 2\Delta \right)\ln{n}$.
%$\Pr\left(d^*_\sigma(v) \leq -\ell_\text{critical}\right) = o(n^{-1})$ for 
%}
\end{proposition}
\begin{proof}
Let $X_i\sim \text{Bernoulli}(p={a\ln{n}}/{n})$ and $Y_i\sim\text{Bernoulli}(q={b\ln{n}}/{n})$. Define $Z_i = X_i - Y_i$:
\begin{equation*}
Z_i =
  \begin{cases}
    +1  & \quad \text{with prob. } p(1-q) \\
     0  & \quad \text{with prob. } pq + (1-p)(1-q)\\
    -1	& \quad \text{with prob. } (1-p)q
  \end{cases}
\end{equation*}
%
%We begin by finding $\ell$ such that $\Pr(d^*_\sigma(v) \leq -\ell) = o(n^{-1})$. 
Note that $d^*_\sigma(v) = \sum_{i=1}^n Z_i$. Therefore,
\begin{align*}
&\Pr(d^*_\sigma(v) \leq -\ell) = \Pr\left(\sum_{i=1}^n Z_i \leq -\ell\right) \\ 
 &= \Pr\left(\sum_{i=1}^n -\left(Z_i-(a-b)\frac{\ln{n}}{n}\right) > \ell+(a-b)\ln{n}\right). \nonumber
\end{align*}
Applying Bernstein inequality to the sum of r.v.'s $\tilde{Z_i} \coloneqq -\left(Z_i-(a-b)\frac{\ln{n}}{n}\right)$, we get $\Pr(d^*_\sigma(v) \leq -\ell) = o(n^{-1})$ if
\begin{equation*}
\ell \geq \ell_\text{critical} = \left(\sqrt{2(a+b)+1/9}+1/3-(a-b)\right)\ln{n}.
\end{equation*}
\end{proof}
\noindent Now, $\Pr(d^*_\sigma(v) \leq \ell_\text{critical})$ can be similarly bounded using Bernstein inequality to get~\eqref{eq:delta} (note that, in order to apply Bernstein inequality, we need to assume $4\Delta>1/3+\sqrt{4M+1/9}$).
\begin{align}
&\Pr(d^*_\sigma(v) \leq \ell_\text{critical}) \leq n^{-\delta}, \label{eq:delta}\\
\text{ with } \delta = &
\frac{\left(4\Delta-1/3-\sqrt{4M+1/9}\right)^2}{2\left(2M+4\Delta-1/3-\sqrt{4M+1/9}\right)}. \nonumber
\end{align}

For sufficiently small $\epsilon$,\footnote{Specifically, we need $\epsilon < \frac{\ell_\text{critical}}{\sqrt{a}\ln{n}}=\frac{\left(\sqrt{2(a+b)+1/9}+1/3-(a-b)\right)}{\sqrt{a}}$. In order to make $\epsilon$ smaller, one has to use more partitions in the replica step (see~\cite{Mossel-14} for details.)}
$d^*_{\sigma}(v) \leq \ell_\text{critical}$ for all $v\in V_\epsilon$. Therefore, sampling the nodes with $d^*_{\tau'}(v) \leq \ell_\text{critical}$ will capture all the nodes in $V_\epsilon$ and thus, correct all the mistakes in the initial labeling $\tau'$.
A bound on the number of samples required is given by $n\cdot\Pr(d^*_\sigma(v) \leq \ell_\text{critical})$.
Thus, sampling $n^{1-\delta}$ nodes with the proposed algorithm is sufficient for exact recovery.

\section{Necessary Condition}
\label{sec:nec_cond}
%
%We consider only the case of $D< 1$, since otherwise it is clear that no sampling is necessary. 
To prove the necessary condition of Theorem~\ref{th:main_nec}, we assume that $s<1-D$, and show that exact recovery is impossible. The proof follows the techniques used in~\cite{Abbe-15} for the necessary condition in clustering.
The proof considers the recovery of minority nodes w.r.t. $\sigma$. A minority node is a node for which the number of neighbors of the same community is smaller than the number of neighbors of the opposite community. The idea of the proof is to show that if there exists a minority node that was not sampled, then no algorithm can estimate the label of the minority node with vanishing error probability. Then we show, using simple counting, that the number of minority nodes is larger than the budget w.h.p., and therefore there must exist at least one unsampled minority node.  

Consider estimation of label $\sigma_v$ of a single node $v$. If the estimator knows the correct labeling of the rest of the nodes in the graph, then its error probability decreases. Therefore we assume that the estimator possesses this knowledge, to obtain a lower bound on the error probability. Next we notice that the estimator with the lowest error probability is the maximum a posteriori (MAP) estimator. Let $X$ be a random variable that represents the number of neighbors of $v$ of community $1$. Similarly let $Y$ be a random variable that represents the number of neighbors of $v$ of community $-1$. Given the observation of $X$ and $Y$, the MAP estimator selects the label 
$$\tau_v^{\star}=\argmax_{\tau_v}\mathbb{P}\{\sigma_v=\tau_v|X=x,Y=y\}.$$ 
Since the labels are a priori equiprobable, the MAP estimator is equal to the Maximum Likelihood (ML) estimator 
$$\tau_v^{\star}=\argmax_{\tau_v}\mathbb{P}\{X=x,Y=y|\sigma_v=\tau_v\}.$$ 
By the definition of the graph, we see that the ML estimator always misclassifies a minority node.

% AG: NEED TO CHECK THIS PART WITH EYAL
It remains to show that w.h.p. the number of minority nodes exceeds the sampling budget. To do this, we use the fact that the probability that a node is a minority is $\Omega(n^{-D}/\ln(n))$~\cite{Abbe-15,Mossel-14}.
With probability $1-o(1)$, the size of each community is within $\sqrt{n}\ln n$ of $n$. Assume that this holds. We let $U$ be a random set of $2n/\ln^3(n)$ nodes of $G$. With probability $1-o(1)$, the number of nodes in $U$ in each community is within $\sqrt{n}$ of $n/\ln^3 n$, and a randomly selected node in $U$ is adjacent to another node in $U$ with probability $o(1)$.

We say that a node in $U$ is ambiguous if it has a minority with respect to the nodes that are not $U$ (i.e., not considering neighbors in $U$). 
The number of neighbors of a node in $V \setminus U$ in each community follow a multivariate Poisson distribution approximately.
%The probability of a node's number of neighbors in $G\setminus U$ in each community is approximately a multivariate Poisson distribution.
Furthermore, it is shown in~\cite{Abbe-15} that the probability that a node is ambiguous is
$$\Omega(n^{-D}/\ln(n)).$$
By the assumption that $s<1-D$, it follows that there exists $\epsilon>0$ such that a node in $U$ is ambiguous with probability $\Omega(n^{\epsilon+s-1})$. For a fixed community assignment and choice of $U$, there is no dependence between whether or not any two nodes are ambiguous. Also, an ambiguous node is not adjacent to any other node in $U$ with probability $1-o(1)$. So, for any choice of $n^s$ nodes to sample, with probability $1-o(1)$ there is at least one unsampled ambiguous node, which is misclassified by an optimal classifier. Therefore, any classifier misclassifies at least one node with probability $1-o(1)$. 

\section{Experiments}
%\subsection{Demonstration of the result}
We generate random graphs using SBM with different values of $D$. $b$ is set to $2$ and $D$ is varied from $0.3$ to $0.6$. For each value of $D$, graphs with number of nodes, $n$,  changing from $1000$ to $2000$ are generated. For each $D$ and $n$, we compute the following quantities (averaged over $30$ trials):
\begin{align}
n_m &= \text{ number of minorities} \nonumber \\
n_s &= \text{ number of nodes with $d^*_{\tau'}(v) \leq d^*_{\tau',\mathrm{err},\max}$} 
\label{eq:quants}
\end{align}
$d^*_{\tau',\mathrm{err},\max}$ denotes the maximum differential degree w.r.t. the initial label predictions $\tau'$ among the nodes, where $\tau'_v \neq \sigma_v$. Note that sampling $n_s$ nodes with the proposed algorithm will capture all the errors in the initial labeling $\tau'$. Number of minorities $n_m$ can be thought of as a necessary lower bound on the number of samples required according to Theorem~\ref{th:main_nec}.
\begin{figure}
\centering
\includegraphics[width=0.34\textwidth]{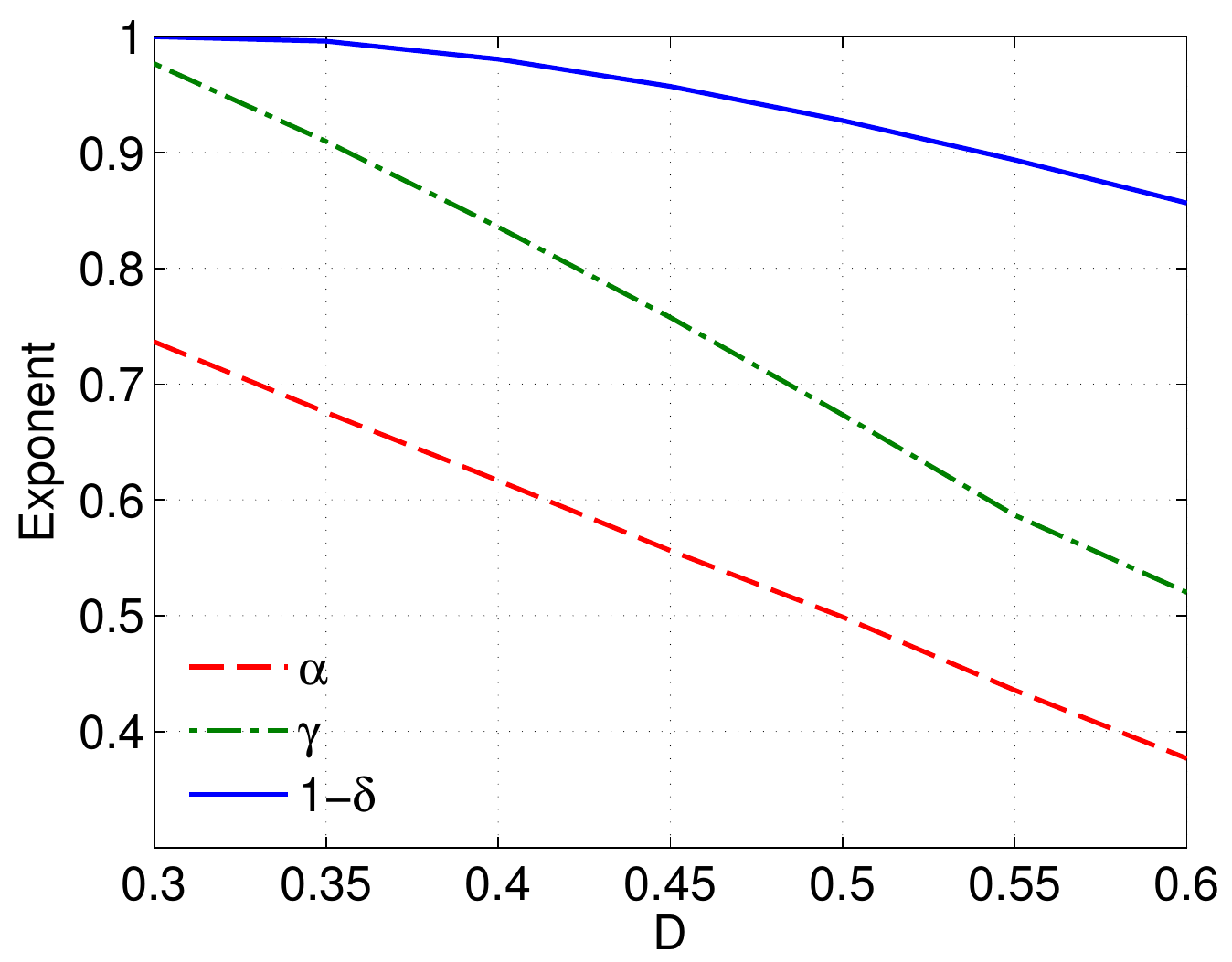}
\caption{Exponents in \eqref{eq:quants_fit} as function of $D$ and the theoretical bound on $\gamma$ given by Theorem~\ref{th:main_suf}.}
%(given by $5/3.(1-D)$).}
\vspace{-1\baselineskip}
\label{fig:exponent_vs_D}
\end{figure}
For each value of $D$, we express the quantities in \eqref{eq:quants} as:
\begin{align}
n_m &\approx c_1 n^\alpha \nonumber \\
n_s &\approx c_2 n^\gamma.
\label{eq:quants_fit}
\end{align}
The exponents $\alpha$ and $\gamma$ are computed by fitting a linear function of $\log{n}$ to the observed values of $\log n_m$ and $\log n_s$ for each value of $D$. (Note that the linear regression needs to be done jointly for different values $D$ in order to keep $c_1$ and $c_2$ fixed as $D$ changes. If linear regression produces $\gamma > 1$, then it is set to $1$, since it implies sampling all nodes.).
Figure~\ref{fig:exponent_vs_D} plots the exponents as a function of $D$ along with the theoretical bound $1-\delta$ given by Theorem~\ref{th:main_suf}. It shows that the exponent in the expression for $n_m$, $\alpha \approx 1-D$, which is consistent with the theoretical value given in \cite{Mossel-14} and thus, acts as a sanity check for the results. Figure~\ref{fig:exponent_vs_D} also verifies that the exponent $\gamma$ in the expression for $n_s$ is indeed less than the sufficient theoretical bound given by Theorem~\ref{th:main_suf}. The experimentally obtained values of $\gamma$ seem to be much lower than the theoretical bound $1-\delta$, suggesting that a potentially tighter bound for $n_s$ can be proven.

We also plot the error rate vs. the fraction of nodes sampled in Figure~\ref{fig:err_vs_num_samp}, for graphs of sizes $n = 1000, 2000$ and $4000$ keeping $D$ fixed at $0.5$. The figure shows that the a smaller fraction of the nodes needs to sampled to achieve the same error as the graph size increases. The fraction of nodes needed to achieve zero error is well below the theoretical bound 
% of Theorem~\ref{th:main_suf} 
and decreases as $n$ increases. This is expected since the number of samples needed for exact recovery is a sub-linear function of $n$ according to Theorem~\ref{th:main_suf}.
% AG: SANITY CHECK: EXPONENT OF NUMBER OF MINORITIES IS 1-D AS EXPECTED. GAMMA IS BETTER THAN THE BOUND GIVEN BY THEOREM 1. SO POTENTIALLY A BETTER BOUND CAN BE PROVED.
%
\begin{figure}
\centering
\includegraphics[width=0.34\textwidth]{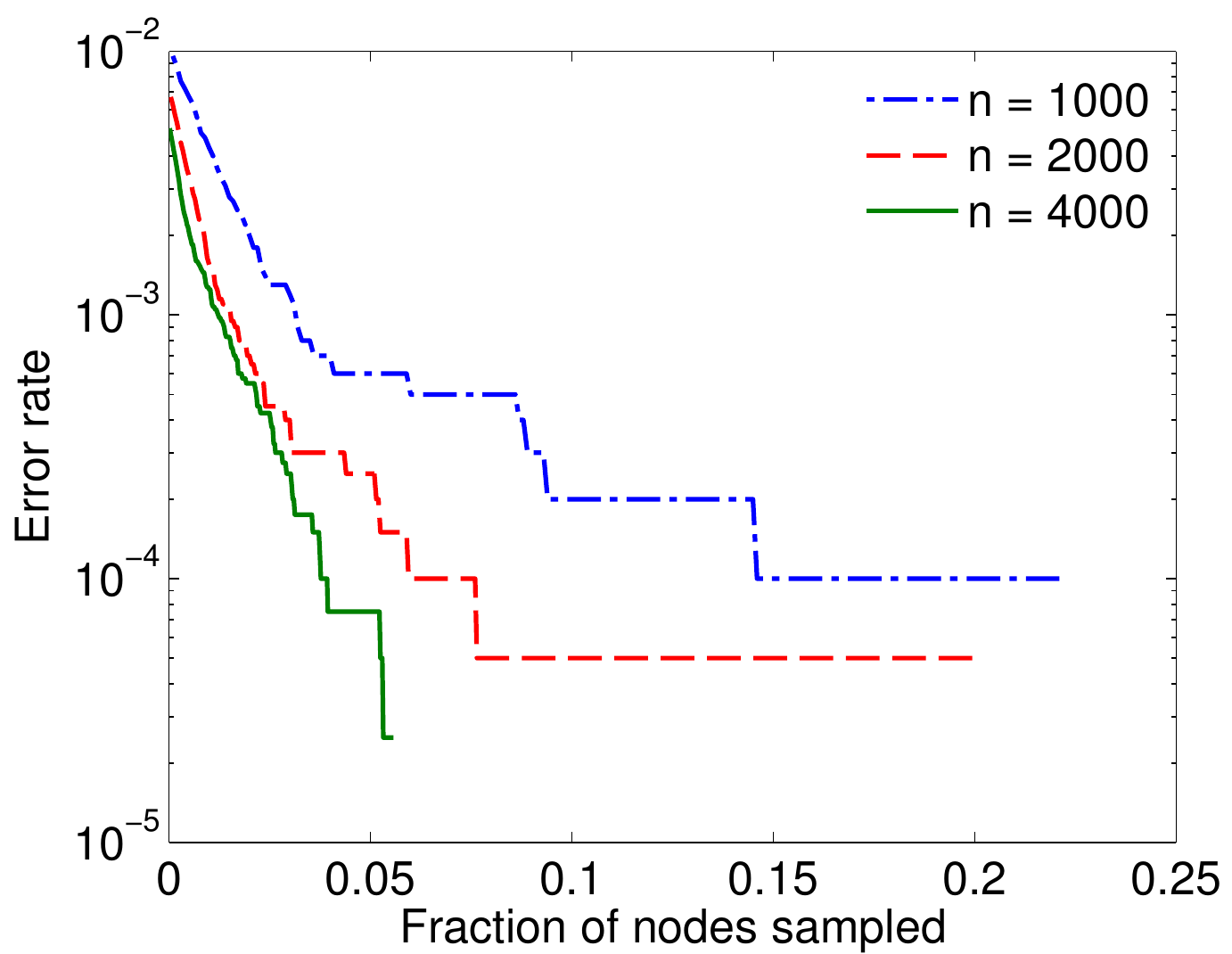}
\caption{Error rate vs. fraction of nodes sampled at $D = 0.5$ for different graph sizes $n$. The theoretical bound on the fraction of samples needed is $0.56$, $0.54$ and $0.51$ for $n = 1000, 2000$ and $4000$ respectively.}
%Circles denote the required fraction of nodes as given by Theorem~\ref{th:main_suf}.}
%(given by $5/3.(1-D)$).}
\vspace{-1\baselineskip}
\label{fig:err_vs_num_samp}
\end{figure}

\section{Conclusion}
In this paper, we considered the problem of active learning for community detection in SBM. Community detection, using only the graph, is possible if and only if the relative difference between the within-community and between-community connection probabilities is larger than certain threshold. We showed that community detection is possible below this threshold by observing the labels of a 
%vanishingly 
small fraction of carefully selected nodes in addition to the graph. We gave the necessary and sufficient conditions on the number of samples required and also provided an algorithm that specifies which nodes should be sampled. 
We also verified our results through numerical experiments. 
%We also illustrated the effectiveness of the proposed sampling method in a standard active learning setting, where the nodes represent points in a Euclidean space.  
%
In future, we would like to generalize our result to an SBM with more than two and possibly overlapping communities. Another interesting extension is to consider community detection and active learning in random graph models with geometry.

%\section*{Acknowledgment}

\bibliographystyle{IEEEtran}
\bibliography{refs}
\end{document}

%% file: macros.tex
\usepackage{bm}

\long\def\comment#1{}

\newfont{\bbb}{msbm10 scaled 700}

\newfont{\bb}{msbm10 scaled 1100}

\newcommand{\Gc}{{\cal G}}

\newtheorem{theorem}{Theorem}

\newtheorem{proposition}{Proposition}

\newtheorem{definition}{Definition}

\newtheorem*{question*}{Question}

\DeclareMathOperator*{\argmax}{arg\,max}